\def\year{2018}\relax
\documentclass[letterpaper]{article} 
\usepackage{aaai18}  
\usepackage{times}  
\usepackage{helvet}  
\usepackage{courier}  
\usepackage{url}  
\usepackage{graphicx}  
\usepackage{booktabs}
\usepackage{amsmath, amssymb, amsthm}

\newtheorem{thm}{Theorem}
\newtheorem{lem}[thm]{Lemma}
\newtheorem{prop}[thm]{Proposition}

\usepackage{color}

\usepackage{algorithm}
\usepackage{algorithmic}

 \newcommand\dl[1]{}
 \newcommand\s[1]{}

\newcommand{\namecite}[1]{\citeauthor{#1}~\shortcite{#1}}

\newcommand\citep{\cite}
\newcommand\citet{\namecite}

\begin{document}
%
\title{Deterministic Policy Optimization by Combining Pathwise\\ and Score Function Estimators for Discrete Action Spaces}
\author{Daniel Levy, Stefano Ermon \\
Department of Computer Science, Stanford University \\
\texttt{\{danilevy,ermon\}@cs.stanford.edu}
}
\maketitle

\hyphenation{REINFORCE}

\begin{abstract}
Policy optimization methods have shown great promise in solving complex reinforcement and imitation learning tasks. While model-free methods are broadly applicable, they often require many samples to optimize complex policies. Model-based methods greatly improve sample-efficiency but at the cost of poor generalization, requiring a carefully handcrafted model of the system dynamics for each task. Recently, hybrid methods have been successful in trading off applicability for improved sample-complexity. However, these have been limited to continuous action spaces. In this work, we present a new hybrid method based on an approximation of the dynamics as an expectation over the next state under the current policy. This relaxation allows us to derive a novel hybrid policy gradient estimator, combining score function and pathwise derivative estimators, that is applicable to discrete action spaces. We show significant gains in sample complexity, ranging between $1.7$ and $25\times$, when learning parameterized policies on Cart Pole, Acrobot, Mountain Car and Hand Mass. Our method is applicable to both discrete and continuous action spaces, when competing pathwise methods are limited to the latter. 
\end{abstract} 

\section{Introduction}

Reinforcement and imitation learning using deep neural networks have achieved impressive results on a wide range of tasks spanning manipulation \cite{levine2016end,levine2014learning}, locomotion \cite{silver2014deterministic}, games \cite{mnih2015human,silver2016mastering}, and autonomous driving~\cite{ho2016generative,li2017inferring}. Model-free methods search for optimal policies without explicitly modeling the system's dynamics~\cite{williams1992simple,schulman2015trust}. Most model-free algorithms build an estimate of the policy gradient by sampling trajectories from the environment and perform gradient ascent. However, these methods suffer from either very high sample complexity due to the generally large variance of the estimator, or are restricted to policies with few parameters.

On the other hand, model-based reinforcement learning methods learn an explicit model of the dynamics of the system. A policy can then be optimized under this model by back-propagating the reward signal through the learned dynamics. While these approaches can greatly improve sample efficiency, the dynamics model needs to be carefully hand-crafted for each task. Recently, hybrid approaches, at the interface of model-free and model-based, have attempted to balance sample-efficiency with generalizability, with notable success in robotics \cite{levine2016end}.


Existing methods, however, are limited to continuous action spaces \cite{levine2014learning,heess2015learning}. 
In this work, we introduce a hybrid reinforcement learning algorithm for deterministic policy optimization that is applicable to continuous \emph{and} discrete action spaces. Starting from a class of deterministic policies, we relax the corresponding policy optimization problem to one over a carefully chosen set of stochastic policies under approximate dynamics. This relaxation allows the derivation of a novel policy gradient estimator, which combines pathwise derivative and score function estimator. This enables incorporating model-based assumptions while remaining applicable to discrete action spaces. We then perform gradient-based optimization on this larger class of policies while slowly annealing stochasticity, converging to an optimal deterministic solution. We additionally justify and bound the dynamics approximation under certain assumptions. Finally, we complement this estimator by a scalable method to estimate the dynamics, first introduced in \cite{levine2014learning}
, and with a general extension to non-differentiable rewards
, rendering it applicable to a large class of tasks. 
Our contributions are as follows:
\begin{itemize}
\item We introduce a novel deterministic policy optimization method, that leverages a model of the dynamics and can be applied to any action space, whereas existing methods are limited to continuous action spaces. We also provide theoretical guarantees for our approximation.
\item We show how our estimator can be applied to a broad class of problems by extending it to discrete rewards, and utilizing a sample-efficient dynamics estimation method 
\cite{levine2014learning}.
\item We show that this method successfully optimizes complex neural network deterministic policies without additional variance reduction techniques. We present experiments on tasks with discrete action spaces where model-based or hybrid methods are not applicable. On these tasks, we show significant gains in terms of sample-complexity, ranging between $1.7$ and $25\times$. 
\end{itemize}

\section{Related Work}

Sample efficiency is a key metric for RL methods, especially when used on real-world physical systems. Model-based methods improve sample efficiency at the cost of defining and learning task-specific dynamics models, while model-free methods typically require significantly more samples but are more broadly applicable.

Model-based methods approximate dynamics using various classes of models, spanning from gaussian processes \cite{deisenroth2011pilco} to neural networks \cite{fu2016one}. Assuming the reward function is known and differentiable, the policy gradient can be computed exactly by differentiating through the dynamics model, thus enabling gradient-based optimization. \citet{heess2015learning} extends this idea to stochastic dynamics and policies by expressing randomness as an exogenous variable and making use of the ``re-parameterization trick". This requires differentiability of the dynamics function w.r.t. both state and action, limiting its applicability to continuous action spaces. On the other hand, model-free algorithms are broadly applicable but require significantly more samples as well as variance reduction techniques through control variates \cite{mnih2016asynchronous,ho2016model} or trust regions \cite{schulman2015trust} to converge in reasonable time.

Recently, hybrid approaches have attempted to balance sample efficiency with applicability. \citet{levine2014learning} locally approximate dynamics to fit locally optimal policies, which then serve as supervision for global policies. A good dynamics model also enables generating artificial trajectories to limit strain on the simulator\s{??}\dl{I don't think I undestood, could you clarify?}~\cite{gu2016continuous}. 
However, these works are once again limited to continuous action spaces. Our work can also be considered a hybrid algorithm that can handle and improve sample-efficiency for discrete action spaces as well.

\section{Background}\label{sec:background}

In this section, we first present the canonical RL formalism. We then review the score function and pathwise derivative estimators, and their respective advantages. We show applications of these with first, the REINFORCE estimator \cite{williams1992simple}, followed by a standard method for model-based policy optimization consisting of back-propagating through the dynamics equation.

\paragraph{Notations and definitions} Let $\mathcal{X}$ and $\mathcal{A}$ denote the state and action spaces, respectively. $f$ refers to a deterministic dynamics function i.e. $f:\mathcal{X} \times \mathcal{A} \to \mathcal{X}$. A deterministic policy is a function $\pi : \mathcal{X} \to \mathcal{A}$, and a stochastic policy is a conditional distribution over actions given state, denoted $\pi(a|x)$. For clarity, when considering parameterized policies, stochastic ones will be functions of $\phi\in\Phi$, and deterministic ones functions of $\theta\in\Theta$. Throughout this work, dynamics will be considered deterministic.

We consider the standard RL setting where an agent interacts with an environment and obtains rewards for its actions. The agent's goal is to maximize its expected cumulative reward. 
Formally, there exists 
an initial distribution $p_0$ over $\mathcal{X}$ and a collection of reward functions $(r_t)_{t\leq T}$ with $r_t:\mathcal{X} \to \mathbf{R}$. 
$x_0$ is sampled from $p_0$, and at each step, the agent is presented with $x_t$ and chooses an action $a_t$ according to a policy $\pi$. 
$x_{t+1}$ is then computed as $f(x_t, a_t)$. In the finite horizon setting, the episode ends when $t=T$. 
The agent is then provided with the cumulative reward, $r(x_0, \ldots, x_T) = \sum_{t\leq T} r_t(x_t)$. The goal of the agent is to find a policy $\pi$ that maximizes the expected cumulative reward $J(\pi) = \mathbf{E}_{x_0, a_0, \ldots, a_{T-1}}\left[ \sum_{t\leq T} r_t(x_t) \right]$.

\subsection{Score Function Estimator and Pathwise Derivative}

We now review two approaches for estimating the gradient of an expectation. Let $p(x;\eta)$ be a probability distribution over a measurable set $\mathcal{X}$ and $\varphi:\mathcal{X}\to \mathbf{R}$. We are interested in obtaining an estimator of the following quantity: $\nabla_\eta \mathbf{E}_{x\sim p(x;\eta)} \left[ \varphi(x) \right]$.

\paragraph{Score Function Estimator} The score function estimator relies on the `log-trick'. It relies on the following identity (given appropriate regularity assumptions on $p$ and $\varphi$):
\begin{equation}
\nabla_\eta \mathbf{E}_{x\sim p(x;\eta)} \left[ \varphi(x) \right] = \mathbf{E}_{x\sim p(x;\eta)} \left[ \varphi(x) \nabla_\eta \log p(x;\eta) \right]
\end{equation}

This last quantity can then be estimated using Monte-Carlo sampling. This estimator is very general and can be applied even if $x$ is a discrete random variable, as long as $\log p(x;\eta)$ is differentiable w.r.t. $\eta$ for all $x$. 
However, it suffers from high-variance \cite{glasserman2013monte}. 

\paragraph{Pathwise Derivative} 
The pathwise derivative estimator depends on $p(x;\eta)$ being re-parameterizable, i.e., there exists a function $g$ and a distribution $p'$ (independent of $\eta$) such that sampling $x\sim p$ is equivalent to sampling $\epsilon \sim p'$ and computing $x = g(\eta, \epsilon)$. Given that observation, $\nabla_\eta \mathbf{E}_{x\sim p(x;\eta)} \left[ \varphi(x) \right] = \nabla_\eta \mathbf{E}_{\epsilon \sim p'(\epsilon)} \left[ \varphi\left(g(\eta, \epsilon)\right) \right] = \mathbf{E}_{\epsilon \sim p'(\epsilon)} \left[ \nabla_\eta \varphi(g(\eta, \epsilon)) \right]$. This quantity can once again be estimated using Monte-Carlo sampling, but is conversely lower variance \cite{glasserman2013monte}. This requires $\varphi(g(\eta, \epsilon))$ to be a differentiable function of $\eta$ for all $\epsilon$.

We override $J(\pi_\theta)$ (resp. $J(\pi_\phi)$) as $J(\theta)$ (resp. $J(\phi)$), and aim at maximizing this objective function using gradient ascent.
\subsection{REINFORCE}


Using the score function estimator, we can derive the REINFORCE rule \cite{williams1992simple}, which is applicable without assumptions on the dynamics or action space. With $\pi_\phi$ a stochastic policy, we want to maximize $J(\phi) = \mathbf{E}_{x_0, a_0, \ldots, a_T}\left[ \sum_{t \leq T} r_t(x_t) \right]$ where $a_t \sim \pi_\phi(\cdot|x_t)$.
Then the REINFORCE rule is $\nabla_\phi J(\phi) = \mathbf{E}_{x_0, a_0, \ldots, a_T}\left[ \sum_{t \leq T} \left(\sum_{t' \geq t} r_{t'}(x_t') \right)\nabla_\phi \log \pi_\phi (a_t | x_t) \right]$. We can estimate this quantity using Monte-Carlo sampling. This only requires differentiability of $\pi_\phi$ w.r.t. $\phi$ and does not assume knowledge of $f$ and $r$. This estimator is however not applicable to deterministic policies.
\subsection{A Method for Model-based Policy Optimization}

Let $\pi_\theta$ be a deterministic policy, differentiable w.r.t. both $x$ and $\theta$. Assuming we have knowledge of $f$ and $r$, 
we can directly differentiate $J$: $\nabla_\theta J = \mathbf{E}_{x_0}\left[\sum_{t\leq T} \nabla_x r_t(x_t) \cdot \nabla_\theta x_t \right]$. The first term, 
$\nabla_x r_t$, is easily computed given knowledge of the reward functions. The second term, given knowledge of the dynamics, can be computed by recursively differentiating $x_{t+1} = f\left(x_t, \pi_\theta(x_t)\right)$, i.e. $\nabla_\theta x_0 = 0$ and:
\begin{equation}
\nabla_\theta x_{t+1} = \nabla_x f \cdot \nabla_\theta x_t + \nabla_a f \cdot \nabla_\theta \pi_\theta(x_t)
\end{equation}


This method can be extended to stochastic dynamics and policies by using a pathwise derivative estimator i.e. by re-parameterizing the noise \cite{heess2015learning}.

This method is applicable for a deterministic and differentiable policy w.r.t. both $x$ and $\theta$
, differentiable dynamics w.r.t. both variables and differentiable reward function. This implies that $\mathcal{A}$ must be continuous. This model-based method 
aims at utilizing the dynamics and differentiating the dynamics equation. 

\section{Relaxing the Policy Optimization Problem}

Deterministic policies are optimal in a non-game theoretic setting\s{cite}\dl{I thought this was well-known/trivially true, will try to find a good ref}. Our objective, in this work, is to perform hybrid policy optimization for deterministic policies for both continuous and discrete action spaces. In order to accomplish this, we present a relaxation of the RL optimization problem. This relaxation allows us to derive, in the subsequent section, a policy gradient estimator that differentiates approximate dynamics and leverages model-based assumptions for any action spaces. Contrary to traditional hybrid methods, \emph{this does not assume differentiability of the dynamics function w.r.t. the action variable, thus elegantly handling discrete action spaces}. As in the previous section, we place ourselves in the finite-horizon setting. 
\s{refer to previous eq.}\dl{I don't think I understand}
We assume terminal rewards and convex state space $\mathcal{X}$.

\subsection{Relaxing the dynamics constraint}

In this section, we describe our relaxation. Starting from a class of deterministic policies $D(\Theta)$ parameterized by $\theta\in\Theta$, we can construct a class of stochastic policies $S(\Phi, \Lambda) \supset D(\Theta)$ parameterized by $(\phi\in\Phi,\lambda\in\Lambda)$, that can be chosen as close as desired to $D(\Theta)$ by adjusting $\lambda$. On this extended class, we can derive a low-variance gradient estimator. We first explain the relaxation, then detail how to construct $S(\Phi, \Lambda)$ from $D(\Theta)$, and finally, provide guarantees on the approximation.

Formally, the RL problem with deterministic policy and dynamics can be written as: 
\begin{equation}
\begin{aligned}
& \underset{\pi_\theta\in D(\Theta)}{\text{maximize}}
& & \mathbf{E}_{x_0 \sim p_0}\left[ r(x_T) \right] \\
& \text{subject to}
& & x_{t+1} = f(x_t, \pi_\theta(x_t))
\end{aligned}
\end{equation}

Given that $\pi_\theta$ is deterministic, the constraint can be equivalently rewritten as:
\begin{equation}
x_{t+1} = \mathbf{E}_{a \sim \pi_\theta(\cdot|x_t)} \left[f(x_t, a) \right]
\end{equation}

Having made this observation, we proceed to relaxing the optimization from over $D(\Theta)$ to $S(\Phi, \Lambda)$, with the constraint now being over approximate and in particular differentiable dynamics.  The relaxed optimization program is therefore:
\begin{equation}\label{eq:relaxed pb}
\begin{aligned}
& \underset{\pi_{\phi,\lambda}\in S(\Phi,\Lambda)}{\text{maximize}}
& & \mathbf{E}_{x_0 \sim p_0}\left[ r(x_T) \right] \\
& \text{subject to}
& & x_{t+1} = \mathbf{E}_{a \sim \pi_{\phi,\lambda}(\cdot|x_t)} \left[f(x_t, a) \right]
\end{aligned}
\end{equation}

This relaxation casts the optimization to be over stochastic policies, which allows us to derive a policy gradient in Theorem~\ref{th:rpg}, but under different, approximated dynamics.
We later describe how to project the solution in $S(\Phi,\Lambda)$ back to an element on $D(\Theta)$.

\subsection{Construction of $S(\Phi,\Lambda)$ from $D(\Theta)$}

Here we show how to construct stochastic policies from deterministic ones while providing a parameterized `knob' to control randomness and closeness to a deterministic policy.

\paragraph{Discrete action spaces} The natural parameterization is as a softmax model. However, this requires careful parameterization, in order to ensure that all policies of $D(\Theta)$ are included. We use the deterministic policy as a prior of which we can control the strength. Formally, we choose a class of parameterized functions $\{ g_\psi : \mathcal{X} \to \mathbf{R}^{|\mathcal{A}|}, \psi\in\Psi \}$. 
\s{how is this class chosen? is this supposed to be for any choice of this class?}\dl{any function work since it's un-normalized log prob}
For $\theta,\psi$ and $\lambda\geq 0$, we can define the following stochastic policy $\pi_{\theta,\psi,\lambda}$ s.t. $\pi_{\theta,\phi,\lambda}(a|x) \propto \exp\left[g_\psi(x)_a + \frac{\mathbf{1}_{a=\pi_\theta(x)}}{\lambda}\right]$. We have therefore defined $S(\Phi,\Lambda)$ where $\Phi \triangleq \Theta \times \Psi$ and $\Lambda \triangleq [0, \infty)$. We easily verify that $D(\Theta) \subset S(\Phi,\Lambda)$, as, for any $\theta\in\Theta$, we can choose an arbitrary $\psi\in\Psi$ and define $\phi = (\theta,\psi) \in\Phi$, we then have $\pi_\theta = \lim_{\lambda \to 0} \pi_{\phi, \lambda}$.


\s{this construction is incomplete without specifying what
$\{ g_\psi : \mathcal{X} \to \mathbf{R}^{|\mathcal{A}|}, \psi\in\Psi \}$
is
}\dl{I don't understand why since this can be anything}

 
\paragraph{Continuous action spaces} In the continuous setting, a very simple parameterization is by adding Gaussian noise, of which we can control variance. Formally, given $D(\Theta)$, let $\Phi \triangleq \Theta, \Lambda \triangleq [0,\infty)$ and $S(\Phi, \Lambda) \triangleq \{x \mapsto \mathcal{N}(\pi_\theta(x), \lambda^2I), \theta \in \Theta, \lambda \geq 0\}$. The surjection can be derived by setting $\lambda$ to $0$. More complicated stochastic parameterizations can be easily derived as long as the density remains tractable.

\paragraph{Rounding} We assume that there is a surjection from $S(\Phi,\Lambda)$ to $D(\Theta)$, s.t. any stochastic policy can be made deterministic by setting $\lambda$ to a certain value. For the above examples, the mapping consists of setting $\lambda$ to $0$. 

Similar in spirit to simulated annealing for optimization \cite{kirkpatrick1983optimization}, we optimize over $\phi$, while slowly annealing $\lambda$ to converge to an optimal solution in $D(\Theta)$.

\subsection{Theoretical Guarantees}

Having presented the relaxation, we now provide theoretical justifications, to show, first, that given conditions on the stochastic policy
, a trajectory computed with approximate dynamics under a stochastic policy is close to the trajectory computed with the true dynamics under a deterministic policy. 
We additionally present connections in the case where our dynamics are discretization of a continuous-time system.

\paragraph{Bounding the deviation from real dynamics} In this paragraph, we assume that the action space is continuous. Given the terminal reward setting, the amount of approximation can be defined as the divergence between $\left(x_t\right)_{t\leq T}$, the trajectory from following a deterministic policy $\pi_\theta$, and $\left(\tilde{x}_t\right)_{t\leq T}$, the trajectory corresponding to the approximate dynamics with a policy $\pi_{\phi,\lambda}$. This will allow us to relate the optimal value of the relaxed program with the optimal value of the initial optimization problem.

\begin{thm}[Approximation Bound]\label{th:approximation}
Let $\pi_\theta$ and $\pi_{\phi,\lambda}$ be a deterministic and stochastic policy, respectively, s.t. $\forall x\in\mathcal{X}, \mathbf{E}_{a\sim\pi_{\phi,\lambda}(\cdot|x)}a = \pi_\theta$. Let us suppose that $f$ is $\rho$-lipschitz and $\pi_\theta$ is $\rho'$-lipschitz. We further assume that $\sup_{x}\mathrm{tr}\left[ \mathrm{Var}_{\pi_{\phi,\lambda}(\cdot | x)}a \right] \leq M$ and that $\alpha \triangleq \rho(\rho'+1) < 1$. We have the following guarantee:
\begin{equation}
\mathbf{E}_{x_0} ||x_T - \tilde{x}_T|| \leq \frac{1}{1-\alpha} \sqrt{M}
\end{equation}

Furthermore, if $\left[\pi_\phi(\cdot|x)\right]_{x\in\mathcal{X}}$ are distributions of fixed variance $\lambda^2I$, the approximation converges towards $0$ when $\lambda \to 0$. 
\end{thm}
\begin{proof}
See Appendix.
\end{proof}
We know that solving the relaxed optimization problem will provide an upper-bound on the expected terminal reward from a policy in $D(\Theta)$. Given a $\rho''$-lipschitz reward function, this bound shows that the optimal value of the true optimization program is within $\frac{\rho''}{1-\alpha} \sqrt{M}$ of the optimal value of the relaxed one.

\paragraph{Equivalence in continuous-time systems} The relaxation has strong theoretical guarantees when the dynamics are a discretization of a continuous-time system. With analogous notations as before, let us consider a continuous-time dynamical system: $x(0) = x_0, \dot{x} = f(x(t), a(t))$.

A $\delta$-discretization of this system can be written as $x_{t+1} = x_t + \delta f(x_t, a_t)$. We can thus write the dynamics of the relaxed problem: $x_{t+1} = x_t + \delta\mathbf{E}_{a\sim \pi_{\phi,\lambda}(a|x_t)} \left[f(x_t, a) \right]$.

Intuitively, when the discretization step $\delta$ tends to $0$, the policy converges to a deterministic one. In the limit, our approximation is the true dynamics for continuous time systems. Proposition~\ref{prop:ct} formalizes this idea.

\begin{prop}\label{prop:ct}
Let $(X^\delta_t)_{t\leq T}$ be a trajectory obtained from the $\delta$-discretized relaxed system, following a stochastic policy. Let $x(t)$ be the continuous time trajectory. Then, with probability $1$:
\begin{equation}
X^\delta_t \to x(t)
\end{equation}
\end{prop}
\begin{proof}
See \cite{munos2006policy}.
\end{proof}

\section{Relaxed Policy Gradient}

The relaxation presented in the previous section allows us to differentiate through the dynamics for any action space. In this section, we derive a novel deterministic policy gradient estimator, Relaxed Policy Gradient (RPG), that combines pathwise derivative and score function estimators, and is applicable to all action spaces. We then show how to apply our estimator, using a sample-efficient and scalable method to estimate the dynamics, first presented in \cite{levine2014learning}
. We conclude by extending the estimator to discrete reward functions, effectively making our algorithm applicable to a wide variety of problems. For simplicity, we omit $\lambda$ (the stochasticity) from our derivations and consider it fixed.

\subsection{Estimator}
We place ourselves in the relaxation setting. Letting $\pi_\phi$ be a parameterized stochastic policy, we define $J(\phi) = \mathbf{E}_{x_0,a_0,\ldots,a_{T-1}} \left[r(x_T)\right]$. Our objective is to find $\phi^* = \arg\max_\phi J(\phi)$. To that aim, we wish to perform gradient ascent on $J$.

\begin{thm}[Relaxed Policy Gradient]\label{th:rpg} Given a trajectory $(x_0, a_1, \ldots, x_T)$, sampled by following a stochastic policy $\pi_\phi$, we define the following quantity $\hat{g} = \nabla_x r(x_T) \cdot \nabla_\phi x_T$ where $\nabla_\phi x_T$ can be computed with the recursion defined by $\nabla_\phi x_0 = 0$ and:
\begin{equation}
\begin{aligned}
\nabla_\phi x_{t+1} & = \nabla_x f(x_t, a_t) \cdot \nabla_\phi x_{t}\\
& + f(x_t, a_t)\nabla_\phi \log \pi_\phi(a_t | x_t)\\ 
& + f(x_t, a_t)\nabla_x \log \pi_\phi(a_t | x_t) \cdot \nabla_\phi x_{t}
\end{aligned}
\end{equation}

$\hat{g}$ is an unbiased estimator of the policy gradient of the relaxed problem, defined in Equation~\ref{eq:relaxed pb}.
\end{thm}

\begin{proof}
See Appendix.
\end{proof}

The presented RPG estimator is effectively a combination of pathwise derivative and score function estimators. Intuitively, the pathwise derivative cannot be used for discrete action spaces as it requires differentiability w.r.t. the action. To this end, we differentiate pathwise through $x$ and handle the discrete portion with a score function estimator.

\paragraph{Benefits of pathwise derivative} Gradient estimates through pathwise derivatives are considered lower variance \cite{glasserman2013monte}. In the context of RL, intuitively, REINFORCE suffers from high variance as it requires many samples to properly establish credit assignment, increasing or decreasing probabilities of actions seen during a trajectory based on the reward. Conversely, as RPG estimates the dependency of the state on $\phi$, the gradient can directly adjust $\phi$ to steer the trajectory to high reward regions. 
This is possible because the estimator utilizes a model of the dynamics. While REINFORCE assigns credit indifferently to actions, RPG adjusts the entire trajectory.

However, when examining our expression for RPG, the computation requires the gradient of the dynamics w.r.t. the state, unlike REINFORCE. In the next section, we present a scalable and sample-efficient dynamics fitting method, further detailed in \cite{levine2016end}, to estimate this term. 

\subsection{Scalable estimation of the dynamics}\label{sec:dynamics}

The Relaxed Policy Gradient estimator 
can be computed exclusively from sampled trajectories if provided with the state-gradient of the dynamics $\nabla_x f$. However, this is not an information available to the agent and thus needs to be estimated from samples. We review a method, first presented in \cite{levine2014learning}, that provides estimation of the state-gradient of the dynamics, incorporating information from multiple sampled trajectories. Although the dynamics are deterministic, we utilize a stochastic estimation to account for modeling errors.

Formally, we have $m$ sampled trajectories $\{ \tau_i \}_{i \leq m}$ of the form $\{ x^i_t, a^i_t, x^i_{t+1}\}_{i\leq m, t \leq T}$ from the dynamical system. Our goal is not only to accurately model the dynamics of the system, but also to estimate correctly the state-gradient. This prevents us from simply fitting a parametric function $f_\phi$ to the samples. Indeed, this approach would provide no guarantee that $\nabla_x f_\phi$ is a good approximation of $\nabla_x f$. This is especially important as the variance of the estimator is dependent on the quality of our approximation of $\nabla_x f$ and not the approximation of $f$.

In order to fit a dynamics model, we choose to locally model the dynamics of the discretized stochastic process as $X_{t+1} \sim \mathcal{N}(A_tX_t + B_t a_t + c_t, F_t)$, parameterized by $\left(\{A_t, B_t, C_t\}\right)_{t\leq T}$. 
We choose this approach as it does not model global dynamics but instead a good time-varying local model around the latest sampled trajectories, which corresponds to the term we want to estimate. Under that model, the term we are interested in estimating is then $A_t$. 

While this approach is simple, it requires a large number of samples to be accurate. To greatly increase sample-efficiency, we assume a prior over sampled trajectories. We use the GMM approach of \cite{levine2014learning} which corresponds to softly piecewise linear dynamics. At each iteration, we update our prior using the current trajectories and fit our model. 
This allows us to utilize both past trajectories as well as nearby time steps that are included in the prior. 
Another key advantage of this stochastic estimation is that it elegantly handles discontinuous dynamics models. Indeed, by averaging the dynamics locally around the discontinuities, it effectively smooths out the discontinuities of the underlying deterministic dynamics. We refer to \cite{levine2016end} for more detailed derivations.

\subsection{Extension to discrete rewards}

Prior to this work, estimators incorporating models of the dynamics such as \cite{heess2015learning,levine2014learning} were constrained to continuously differentiable rewards. We present an extension of this type of estimators to a class of non-continuous rewards. To do so, we make assumptions on the form of the reward and approximate it by a smooth function.

We assume that $\mathcal{X} = \mathbf{R}^n$ and that the reward can be written as a sum of indicator functions, i.e.:
\begin{equation}
r(x) = \sum_{i=1}^k \lambda_i \mathbf{1}_{x\in K_i}
\end{equation}
\noindent where $(K_i)_{i\leq k}$ are compact subsets of $\mathbf{R}^n$. This assumption covers a large collection of tasks where the goal is to end up in a compact subspace of $\mathcal{X}$. For each $K_i$, we are going to approximate $\mathbf{1}_{x\in K_i}$ by an arbitrarily close smooth function.

\begin{prop}[Smooth approximation of an indicator function]\label{prop:urysohn}
Let $K$ be a compact of $\mathbf{R}^n$. For any neighborhood $\Omega$ of $K$, there exists, $\psi:\mathbf{R}^n \to \mathbf{R}$, smooth, s.t. $\mathbf{1}_K \leq \psi \leq \mathbf{1}_\Omega$.
\end{prop}
\begin{proof}
See Appendix.
\end{proof}

We can now approximate each $(\mathbf{1}_{K_i})_{i \leq K}$ by a smooth function $(\psi_i)_{i \leq K}$. Given this surrogate reward $\tilde{r} \triangleq \sum_{i=1}^k \lambda_i \psi_i$, we can apply our estimator. In practice, however, discrete reward functions are often defined as $r(s) = \lambda \mathbf{1}_{h(s) \geq 0}$, where $h$ is a function from $\mathcal{X}$ to $\mathbf{R}$. 
We approximate the reward function by $\tilde{r}_\alpha(s) = \lambda \sigma(\alpha h(s))$ ($\sigma$ being the sigmoid function). If $h$ is $\mathcal{C}^1$, then $\tilde{r}$ is too and $\tilde{r}_\alpha$ pointwise converges to $r$. We present in the Appendix an example of the approximate reward functions for the Mountain Car task.

Given these arbitrarily close surrogate rewards, we can now prove a continuity result, i.e. the sequence of optimal policies under the surrogate rewards converges to the optimal policy.
\begin{prop}[Optimal policies under surrogate reward functions]\label{prop:surrogate}
Let $r$ be a reward function defined as $r = \mathbf{1}_K$ and $\pi^*$ the optimal policy under this reward. Let us define $V^\pi(x)$ to be the value of state $x$ under the policy $\pi$. Then, there exists a sequence of smooth reward functions $(r_n)_{n\in\mathbf{N}}$ s.t. if $\pi_n$ is optimal under the reward $r_n$, $\forall x \in \mathcal{X}, \lim_{n\to\infty} V^{\pi_n}(x) = V^{\pi^*}$.
\end{prop}
\begin{proof}
See Appendix.
\end{proof}

\section{Practical Algorithm}\label{sec:practical}

In Algorithm~\ref{alg:dctpg}, we present a practical algorithm to perform deterministic policy optimization based on the previously defined RPG. In summary, given $D(\Theta)$, we construct our extended class of stochastic policies $S(\Phi, \Lambda)$. We then perform gradient-based optimization over $\phi$, i.e. stochastic policies, while slowly annealing $\lambda$, converging to a 
policy in $D(\Theta)$. We easily extend the estimator to the case of non-terminal rewards as:
\begin{equation}
\nabla_\phi J(\phi, \lambda) = \sum_{t\leq T} \nabla_x r(x_t) \cdot \nabla_\phi x_t
\end{equation}
\begin{small}
\begin{algorithm}[h!]
   \caption{Relaxed Policy Gradient}
   \label{alg:example}
\begin{algorithmic}
   \STATE {\bfseries Inputs:} Environment giving $x_{t+1}, \nabla_x r(x_t)$ given $x_t, a_t$, deterministic class of policies $D(\Theta)$, stochastic class of policies $S(\Phi, \Lambda)$, number of training episodes $N_\mathrm{ep}$, learning rate schedule $(\alpha_N)_{N\leq N_\mathrm{ep}}$, annealing schedule $(\gamma_N)_{N\leq N_\mathrm{ep}}$, initial parameters $\phi_0\in\Phi, \lambda_0 \in\Lambda$.
   \STATE Initialize $\phi\gets \phi_0$ and $\lambda\gets\lambda_0$.
   \FOR{$N=1$ to $N_\mathrm{ep}$}
   \STATE Sample $\{ \tau_i \}_{i \leq m}$ with policy $\pi_{\phi,\lambda}$
   \STATE Update GMM prior over dynamics
   \STATE Fit dynamics $(A_t, B_t, C_t)_{t\leq T}$
   \FOR{$i=1$ {\bfseries to} $m$}
   \STATE $\hat{g_i} \gets 0$
   \STATE $\nabla_\phi x_{0} \gets 0$
   \FOR{$t=0$ {\bfseries to} $T-1$}
   \STATE $\hat{g_i} \gets \hat{g_i} + \nabla_x r(x_t) \cdot \nabla_\phi x_{t-1}$
   \STATE $\hat{f}_t = x_{t+1} - x_{t}$
   \STATE $\nabla_\phi x_{t+1} \gets A_t \cdot \nabla_\phi x_{t} + \hat{f}_t\nabla_\phi \log \pi_{\phi,\lambda}(a_t | x_t)$
   \STATE \hspace{2cm} $+ \hat{f}_t \nabla_x \log \pi_{\phi,\lambda}(a_t | x_t) \cdot \nabla_\phi x_{t}$
   \ENDFOR
   \STATE $\hat{g_i} \gets \hat{g_i} + \nabla_x R(x_T) \cdot \nabla_\phi x_T$
   \ENDFOR
   \STATE $\phi \gets \phi + \alpha_N \hat{g}$
   \STATE $\lambda\gets \gamma_N\lambda$
   \ENDFOR
   \STATE {\bfseries Return:} $\pi_{\phi,\lambda=0} \in D(\Theta)$
\end{algorithmic}
\label{alg:dctpg}
\end{algorithm}
\end{small}

\section{Experiments}
We empirically evaluate our algorithm to investigate the following questions: 1) How much does RPG improve sample complexity? 2) Can RPG learn an optimal policy for the true reward using a smooth surrogate reward? 3) How effective is our approximation of the dynamics?

In an attempt to evaluate the benefits of the relaxation compared to other estimators as fairly as possible, we do not use additional techniques, such as trust regions \cite{schulman2015trust}. Both our method and the compared ones can incorporate these improvements, but we leave study of these more sophisticated estimators for future work.

\subsection{Classical Control Tasks}

We apply our Relaxed Policy Gradient (RPG) to a set of classical control tasks with \emph{discrete} action spaces: Cart Pole~\cite{barto1983neuronlike}, Acrobot~\cite{sutton1996generalization}, Mountain Car~\cite{sutton1996generalization} and Hand Mass~\cite{munos2006policy}. For the first three tasks, we used the OpenAI Gym \cite{brockman2016openai} implementation and followed \cite{munos2006policy} for the implementation of Hand Mass. A diagram of our tasks are presented in Figure~\ref{fig:tasks}.

\begin{figure}
\centering
\frame{\includegraphics[height=2.2cm]{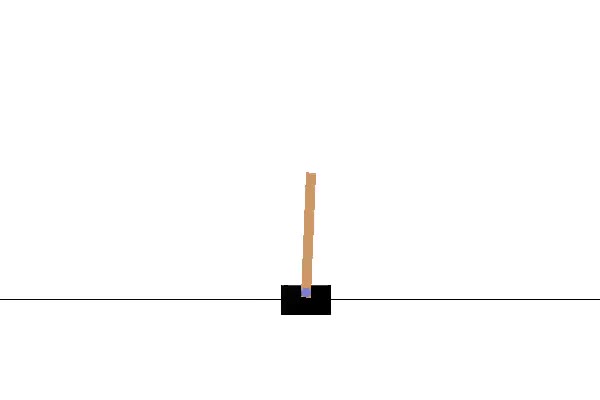}}
\frame{\includegraphics[height=2.2cm]{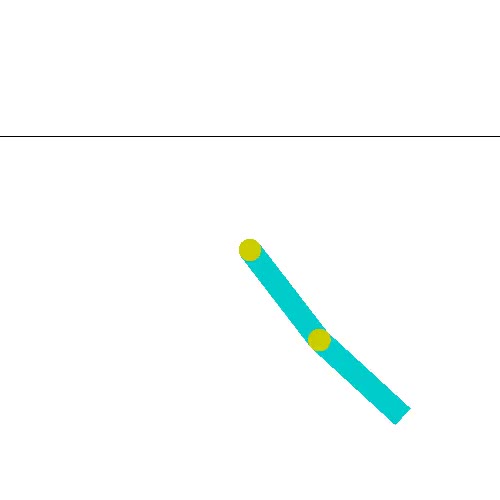}}
\frame{\includegraphics[height=2.2cm]{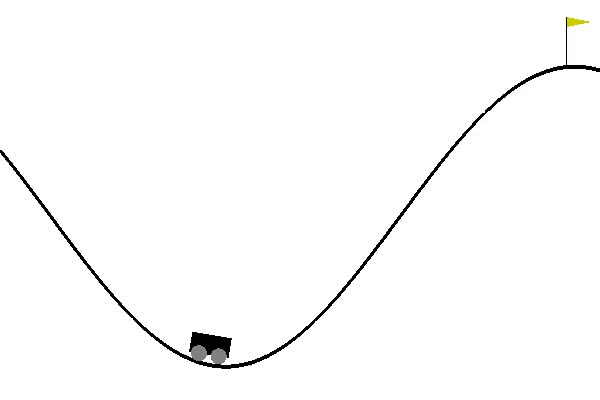}}
\includegraphics[height=2.2cm]{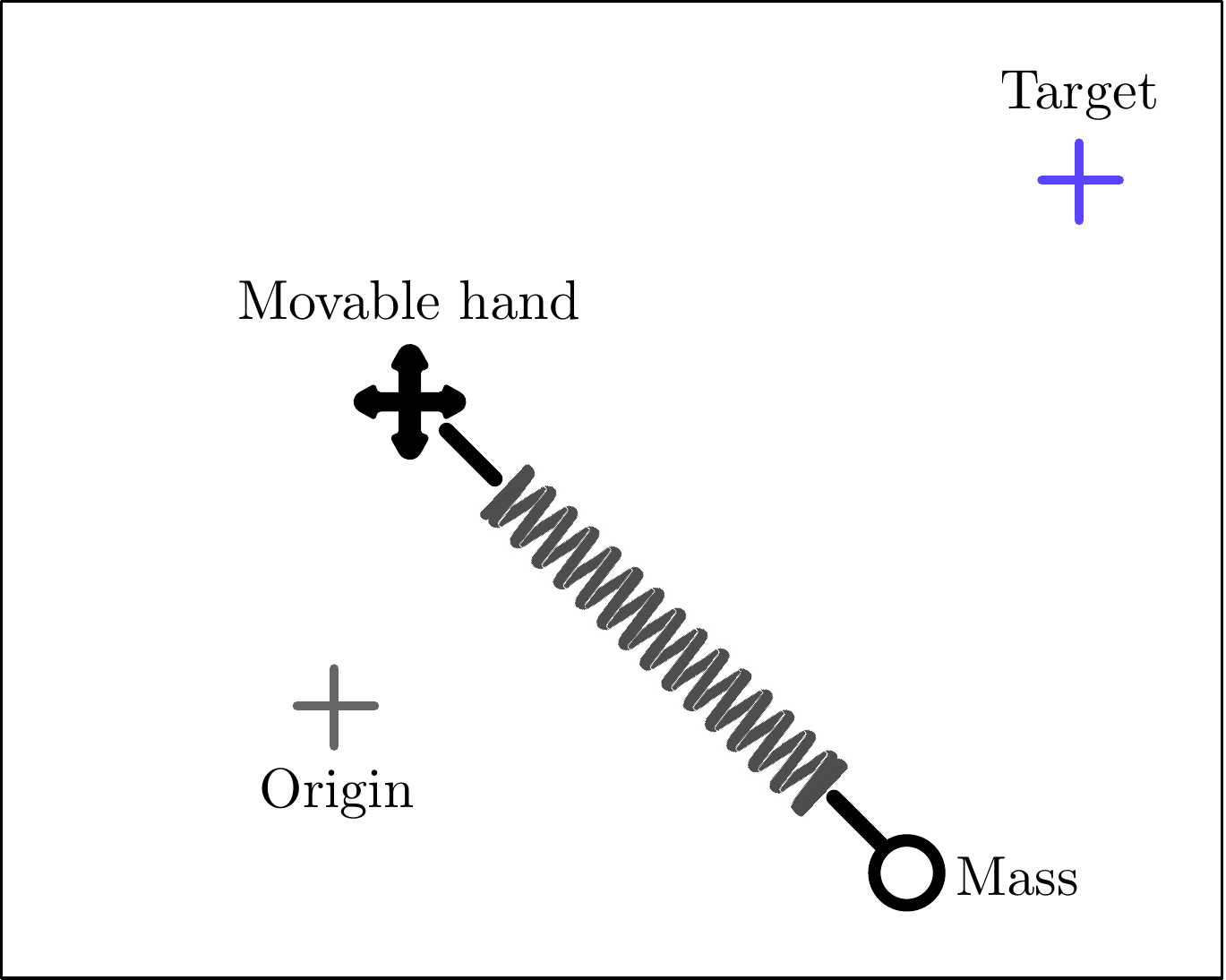}
\caption{\textit{Top row}: Cart Pole, Acrobot. \textit{Bottom row}: Mountain Car, Hand Mass.}
\label{fig:tasks}
\end{figure}

\paragraph{Baselines} We compare our methods against two different baselines: a black-box derivative free optimization algorithm called Cross-Entropy Method (CEM) \cite{szita2006learning} and the Actor-Critic algorithm (A3C) \cite{mnih2016asynchronous}. CEM is an evolutionary-type algorithm that samples a different set of parameters at each step, simulates using these parameters and keeps only the top $2\%$, using those to re-fit the distribution over parameters
. CEM is known to work  well \cite{duan2016benchmarking}, but lacks sample efficiency as the algorithm does not exploit any structure. A3C is a variant of REINFORCE \cite{williams1992simple} that learns a value function to reduce the variance of the estimator. For each task and algorithm, we evaluate for $5$ distinct random seeds.

\subsection{Results and Analysis}
We present the learning curves for all tasks in Figure~\ref{fig:all_curves}. Even when training is done with surrogate rewards, in both instances we report the actual reward. We do not show CEM learning curves as these are not comparable due to the nature of the algorithm; each CEM episode is equivalent to $20$ RPG or A3C episodes\s{why not rescale so that they are comprable?}\dl{it looked weird when I tried back then because we would only see 1/25th of the learning curve}, we instead report final performance in Tables~\ref{tab:samples} and \ref{tab:perfspring}. For all tasks, the policy is parameterized by a $2$-layer neural network with \texttt{tanh} non-linearities and $64$ hidden units, with a softmax layer outputting a distribution over actions. In practice, we optimize over stochastic policies with a fixed $\lambda$. In our experiments, it converged to a near deterministic policy since the optimization is well-conditioned enough in that case. 
The policies are trained using Adam \cite{kingma2014adam}.

\begin{figure*}[ht]
\begin{center}
\includegraphics[width=\textwidth]{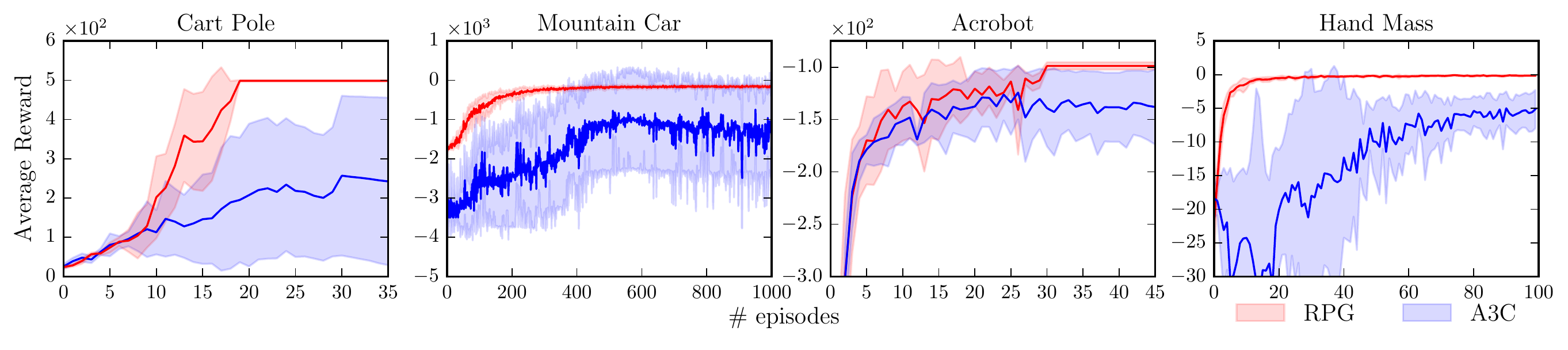}
\end{center}
\vspace{-0.3cm}
\caption{Mean rewards over $5$ random seeds for classical control tasks. Performance of RPG is shown against A3C.}
\label{fig:all_curves}
\end{figure*}

We evaluate differently depending on the task: since Cartpole and Acrobot have a fixed reward threshold at which they are considered solved, for these we present the number of training samples needed to reach that performance (in Table~\ref{tab:samples}). In contrast, for Mountain Car and Hand Mass, we report (in Table~\ref{tab:perfspring}) the reward achieved with a fixed number of samples. Both of these metrics are meant to evaluate sample-efficiency. 

\paragraph{Sample Complexity} As shown in Tables~\ref{tab:samples} and \ref{tab:perfspring}, our algorithm clearly outperforms A3C across all tasks, requiring between $1.7$ and $3$ times less samples to solve the tasks and showing significantly better performance for the same number of samples. As shown in Table~\ref{tab:perfspring}, RPG performs better than CEM in Hand Mass and within $20\%$ for Mountain Car, \emph{despite using $20$ times less samples}. We also note that CEM is particularly well suited for Acrobot, as it is a derivative-free method that can explore the space of parameters very fast and find near-optimal parameters quickly when the optimal policies are fairly simple, explaining its impressive sample-complexity on this specific task.
We additionally point out that full plots of performance against number of samples are reported for all tasks in Figure~\ref{fig:all_curves}, and the numbers in Tables~\ref{tab:samples} and \ref{tab:perfspring} can all be extracted from there.
\paragraph{Robustness, Variance and Training Stability} Overall, our method was very robust to the choice of hyper-parameters such as learning rate or architecture. Indeed, our policy was trained on all tasks with the same learning rate $\alpha = 10^{-2}$, whereas different learning rates were cross-validated for A3C. When examining the training curves, our estimator demonstrates significantly less variance than A3C and a more stable training process. In tasks where the challenges are exploratory (Acrobot or Mountain Car), RPG's exploration process is guided by its model of the dynamics while A3C's is completely undirected, often leading to total failure. The same phenomenon can be observed on control tasks (Hand Mass or Cart Pole), where A3C favors bad actions seen in high reward trajectories, leading to unstable control.

\paragraph{Approximate Reward} On all tasks except Hand Mass, our estimator was trained using approximate smoothed rewards. The performances reported in Figure~\ref{fig:all_curves} and Table~\ref{tab:perfspring} show that this did not impair training at all. Regarding the baselines, it is interesting to note that since CEM does not leverage the structure of the RL problem, it is natural to train with the real rewards. For A3C, we experimented with the smoothed rewards and obtained comparable numbers.

\begin{table}[t]
\begin{center}
\begin{small}
\begin{tabular}{lc}
\toprule
Method & Samples until solve\\
\midrule
\multicolumn{2}{c}{Cart Pole (Threshold = $495$)}\\
\midrule
CEM & $470$ \\
A3C & $279$ \\
RPG & $91$ \\
\midrule
\multicolumn{2}{c}{Acrobot (Threshold = -105)}\\
\midrule
CEM & $180$ \\
A3C & $541$ \\
RPG & $311$ \\
\bottomrule
\end{tabular}
\end{small}
\end{center}
\caption{Average numbers of samples until the task is solved for the Cart Pole and Acrobot tasks for RPG, A3C and CEM.}
\label{tab:samples}
\end{table}

\begin{table}[t]
\begin{center}
\begin{small}
\begin{tabular}{lcc}
\toprule
Method & Samples & Performance\\
\midrule
\multicolumn{3}{c}{Hand Mass ($150$ episodes)}\\
\midrule
CEM & $50$ & $-0.086 \pm 0.01$ \\
A3C & $2$ & $-2.42 \pm 2.19$ \\
RPG & $2$ & $-0.026 \pm 0.01$ \\
\midrule
\multicolumn{3}{c}{Mountain Car ($1000$ episodes)}\\
\midrule
CEM & $200$ & $ -110.3 \pm 1.2$ \\
A3C & $10$ & $ -176 \pm 38.6$ \\
RPG & $10$ & $-131.3 \pm 5.3$ \\
\bottomrule
\end{tabular}
\end{small}
\end{center}
\caption{Average mean rewards for the Hand Mass and Mountain Car tasks for RPG, A3C and CEM.}
\label{tab:perfspring}
\end{table}

\section{Limitations}

In this section, we explore limitations of our method. To leverage the dynamics of the RL problem, we trade-off some flexibility in the class of problems that model-free RL can tackle for better sample complexity. This can also be seen as an instance of the bias/variance trade-off.

\paragraph{Limitations on the reward function} While we presented a general way to extend this estimator for rewards in the discrete domain, such function approximations can be difficult to construct for high-dimensional state-spaces such as Atari games. Indeed, one would have to fit the indicator function of a very low dimensional manifold - corresponding to the set of images encoding the state of a given game score - living in a high-dimensional space (order of $\mathbf{R}^{200 \times 200 \times 3}$). 

\paragraph{Limitations on the type of tasks} While we show results on classical control tasks, our estimator is broadly applicable to all tasks where the dynamics can be estimated reasonably well. This has been shown to be possible on a number of locomotion and robotics tasks \cite{levine2014learning,heess2015learning}. However, our work is not directly applicable to raw pixel input tasks such as the Atari domain. 

\paragraph{Computational overhead} Compared to REINFORCE, our model presents some computational overhead as it requires evaluating $\left(\nabla_\phi x_t\right)_{t\leq T}$ as well as fitting $T$ dynamics matrices. In practice, this is minor compared to other training operations such as sampling trajectories or computing necessary gradients. In our experiments, computing the $\left(\nabla_\phi x_t\right)_{t\leq T}$ amounts to less than $1\%$ of overhead while dynamics estimation constitutes about $17\%$.

\section{Discussion} 

In this work, we presented a method to find an optimal deterministic policy for any action space and in particular discrete actions. This method relies on a relaxation of the optimization problem to a carefully chosen, larger class of stochastic policies. On this class of policies, we can derive a novel, low-variance, policy gradient estimator, Relaxed Policy Gradient (RPG), by differentiating approximate dynamics. We then perform gradient-based optimization while slowly annealing the stochasticity of our policy, converging to a deterministic solution.
We showed how this method can be successfully applied to a collection of RL tasks, presenting significant gains in sample-efficiency and training stability over existing methods. Furthermore, we introduced a way to apply this algorithm to non-continuous reward functions by learning a policy under a smooth surrogate reward function, for which we provided a construction method. It is also important to note that our method is easily amenable to problems with stochastic dynamics, assuming one can re-parameterize the noise.


This work also opens the way to promising future extensions of the estimator. For example, one could easily incorporate imaginary rollouts \cite{gu2016continuous} with the estimated dynamics model or trust regions \cite{schulman2015trust}. Finally, this work can also be extended more broadly to gradient estimation for any discrete random variables.


\paragraph{Acknowledgments} The authors thank Aditya Grover, Jiaming Song and Steve Mussmann for useful discussions, as well as the anonymous reviewers for insightful comments and suggestions. This research was funded by Intel Corporation, FLI, TRI and NSF grants $\#1651565$, $\#1522054$, $\#1733686$.
\bibliographystyle{plain}
\bibliography{sample}

\clearpage
\newpage

\appendix

\section{Appendix}
\subsection{Tasks specifications}\label{app:tasks}
\paragraph{Cart Pole} The classical Cart Pole setting where the agent wishes to balance a pole mounted on a pivot point on a cart. Due to instability, the agent has to perform continuous cart movement to preserve the equilibrium. The agent receives reward $1$ for each time step spend upwards. The task is considered solved if the agent reaches a reward of $495$.

\paragraph{Acrobot} First introduced in \cite{sutton1996generalization}, the system is composed of two joints and two links. Initially hanging downwards, the objective is to actuate the links to swing the end of lower link up to a fixed height. The path's reward is determined by the negative number of actions required to reach the objectif. The task is considered solved if the agent reaches a reward of $-105$.

\paragraph{Mountain Car} We consider the usual setting presented in \cite{sutton1996generalization}. The objective is to get a car to the top of a hill by learning a policy that uses momentum gained by alternating backward and forward motions. The state consists of $(x, \dot{x})$, where $x$ is the horizontal position of the car. The reward can be expressed as $r(x) = \mathbf{1}_{x\geq \frac{1}{2}}$. The task's difficulty lies in both the limitation of maximum tangential force and the need for exploration. The agent receives a reward of $1$ We use a surrogate smoothed reward for DCTPG. We evaluate the performance after a fixed number of episodes.

\paragraph{Hand Mass} We consider a physical system first presented in \cite{munos2006policy}. The agent is holding a spring linked to a mass. The agent can move in four directions. The objective is to bring the mass to a target point $(x_t, y_t)$ while keeping the hand closest to the origin point $(0, 0)$. The state can be described as $s = (x_0, y_0, x, y, \dot{x}, \dot{y})$ and the dynamics are described in \cite{munos2006policy}. At the final time step, the agent receives the reward $r(s) = -(x-x_t)^2 - (y-y_t)^2 -x_0^2 - y_0^2$.
\subsection{Proof of Theorem~\ref{th:approximation}}
\begin{lem}\label{lem:var} Let $X$ be a random variable of $\mathbf{R}^n$ of bounded variance. Let $\Sigma = \mathrm{Var}X$.

$$\mathbf{E}||X - \mathbf{E}X|| \leq \sqrt{\mathrm{tr}\Sigma}$$

\end{lem}

\begin{proof}
We can simply use Cauchy-Schwarz inequality. Let's define $\mu = \mathbf{E}X$.

$$
\begin{aligned}
\mathbf{E}||X - \mathbf{E}X|| & = \int_x ||x-\mu||p(x)\mathrm{d}x \\
& \leq \sqrt{\int_x ||x - \mu||^2 p(x)\mathrm{d}x}\sqrt{\int_x p(x)\mathrm{d}x}\\
& \leq \sqrt{\int_x \sum_i (x_i - \mu_i)^2p(x)\mathrm{d}x}\\
& \leq \sqrt{\sum_i \mathrm{Var}X_i}\\
& \leq \sqrt{\mathrm{tr}\Sigma}
\end{aligned}
$$
\end{proof}
Let's prove this by recursion.

\begin{align}
||x_{t+1} - \tilde{x}_{t+1}|| & = ||\mathbf{E}_{a\sim \pi_\phi(\cdot|\tilde{x}_t)}\left[f(x_t, \pi_\theta(x_t)) - f(\tilde{x}_t, a) \right]|| \\
& \leq \mathbf{E}_{a\sim \pi_\phi(\cdot|\tilde{x}_t)}||f(x_t, \pi_\theta(x_t)) - f(\tilde{x}_t, a)||\\
& \leq \rho ||x_t - \tilde{x}_t|| + \rho\mathbf{E}_{a\sim \pi_\phi(\cdot|\tilde{x}_t)} ||a - \pi_\theta(x_t)|| \\
& \leq \rho ||x_t - \tilde{x}_t|| + \rho\mathbf{E}_{a\sim \pi_\phi(\cdot|\tilde{x}_t)} ||a - \pi_\theta(\tilde{x}_t)|| \nonumber \\ 
& \qquad \qquad \qquad+ \rho||\pi_\theta(\tilde{x}_t) - \pi_\theta(x_t)|| \\
& \leq \rho(\rho' +1) ||x_t - \tilde{x}_t|| + \rho \sqrt{M}
\end{align}

(2) is obtained with triangular inequality. (3) because $f$ is $\rho$-lipschitz. (4) with triangular inequality and because $\mathbf{E}_{a\sim\pi_\phi(\cdot|x)}a = \pi_\theta(x)$. (5) because of Lemma~\ref{lem:var} and the uniform bound on variances.

Given this inequality and the fact that $x_0 = \tilde{x}_0$, we can compute the geometric sum i.e. $$||x_T - \tilde{x}_T|| \leq \frac{1-[\rho(\rho'+1)]^T}{1-\rho(\rho'+1)}\sqrt{M}$$

$\alpha = \rho(\rho'+1) < 1$ concludes the proof.

\subsection*{Proofs of Theorem~\ref{th:rpg}}

As in the model-based method presented in Background, we want to directly differentiate $J(\phi)$. As before, we require the reward function to be differentiable. $\nabla_\phi J = \nabla_x r \cdot \nabla_\phi x_T$, but when computing $\nabla_\phi x_T$, we need to differentiate through the relaxed equation of the dynamics:
$$
\begin{aligned}
\nabla_\phi x_{t+1} & = \nabla_\phi \sum_a f(x_t, a)\pi_\phi(a|x_t) \\
& = \sum_a \left[ \pi_\phi(a|x_t) \nabla_x f \cdot \nabla_\phi x_t \right. \\ 
	& \left. \hspace{35pt} + f(x_t, a)\left(\nabla_\phi \pi_\phi + \nabla_x\pi_\phi \cdot \nabla_\phi x_t\right)\right] \\
& = \mathbf{E}_{\pi_\phi}\left[ \nabla_x f \cdot \nabla_\phi x \right. \\ 
	& \left. \hspace{35pt} + f(x_t, a)\left( \nabla_\phi \log \pi_\phi + \nabla_x \log \pi_\phi \cdot \nabla_\phi x_t \right)\right]
\end{aligned}
$$

The recursion corresponds to a Monte-Carlo estimate with trajectories obtained by running the stochastic policy in the environment, thus concluding the proof.

\subsection{Proof of Proposition~\ref{prop:urysohn}}\label{proof:prop3}
Given a compact $K$ of $\mathbf{R}^n$ and an open neighborhood $\Omega \supset K$, let us show that there exists $\psi \in \mathcal{C}^\infty(\mathbf{R}^n, \mathbf{R})$ s.t. $1_K \leq \psi \leq 1_\Omega$. While according to Urysohn's lemma \cite{alexandroff1924theorie}, such continuous functions exist in the general setting of normal spaces given any pair of disjoint closed sets $F$ and $G$, we will show how to explicitly construct $\psi$ under our simpler assumptions.

It is sufficient to show that we can construct a function $f \in \mathcal{C}^\infty(\mathbf{R}^n, \mathbf{R})$ such that for $\Omega$ open set containing the closed unit ball, $f$ is equal to 1 on the unit ball and $f$ is null outside of $\Omega$. Given this function, we can easily construct $\psi$ by scaling and translating.

Let us define $\phi(x) \triangleq \exp(-\frac{1}{1 - ||x||^2})\mathbf{1}_{||x|| \leq 1}$. This function is smooth on both $\mathcal{B}(0, 1)$ and $\mathcal{B}(0, 1)^C$.


Let us define $F_n(x) \triangleq \int_{\mathcal{B}(0, n)} \phi(t+x)\mathrm{d}\mu(t)$. This function is strictly positive on $\mathcal{B}(0, n)$ and null outside of $\mathcal{B}(0, n+1)$. For $n$ large enough, the function $f(x) = F_n(\frac{x}{n})$ satisfy the desired property, thus proving that we can construct $\psi$.

Figure~\ref{fig:smooth} shows an example of a smooth surrogate reward for the Mountain Car task.

\begin{figure}[bt]
\centering
\includegraphics[width=0.5\textwidth]{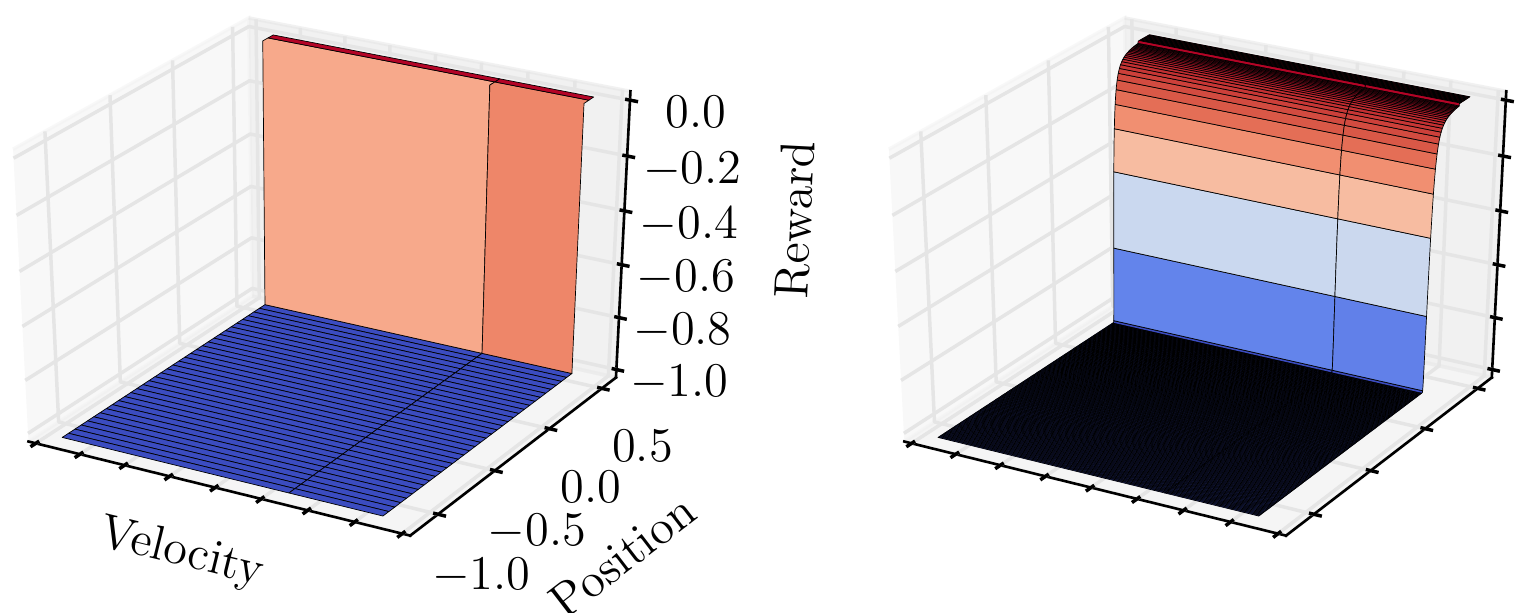}
\caption{Surrogate smooth reward for the Mountain Car task. \textit{Left:} True reward. \textit{Right:} Smoothed surrogate reward.}
\label{fig:smooth}
\end{figure}

\subsection{Proof of Proposition~\ref{prop:surrogate}}\label{proof:prop4}

Let us consider the deterministic case where $r = \mathbf{1}_K$ and the reward is terminal. Let $\pi^*$ be the optimal policy under reward $r$. 

There exists a sequence of open neighborhoods $\left(\Omega_n\right)_{n\in\mathbf{N}}$ s.t. $\Omega_n \subset \Omega_{n+1}$ and $\mathrm{diameter}(\Omega_n - K)$ tends to $0$. Let $r_n$ be the smooth reward function constructed for $\mathbf{1}_K$ with neighborhood $\Omega_n$. Let $\pi_n$ be the deterministic optimal policy under reward $r_n$.

For a start state $x$, we can distinguish between two cases:
\begin{enumerate}
\item $V^{\pi^*}(x) = 0$
\item $V^{\pi^*}(x) = 1$
\end{enumerate}

For the former, no policy can reach the compact and thus the surrogate reward make no difference. For the latter, the policy $\pi_n$, take $x$ to final state $x^n_T$ where $x^n_T \in \Omega_n$. However, $\mathrm{d}(x^n_T, K) \leq \mathrm{diameter}(\Omega_n - K)$ which tends to $0$. The $(x^n_T)_{n\in\mathbf{N}}$ is a Cauchy sequence in a complete space and thus converges to a limit $x^*_T$. Since $\lim_{n\to\infty} \mathrm{d}(x^n_T, K) = 0$ and by continuity of the distance to a compact, $x^*_T \in K$ proving the result.

\end{document}